\documentclass{article}

\usepackage[round]{natbib}

\usepackage{fullpage}
\usepackage{amsthm}
\usepackage{listings}
\usepackage{xcolor}
\definecolor{RoyalBlue}{RGB}{30,150,230}
\usepackage[colorlinks = true,
            linkcolor = blue,
            urlcolor  = blue,
            citecolor = blue,
            anchorcolor = blue]{hyperref}
\usepackage{url}
\usepackage{enumitem}
\usepackage[bottom]{footmisc}

\usepackage{amsmath,amsfonts,bm}

\def\eqref#1{equation~\ref{#1}}

\def\floor#1{\lfloor #1 \rfloor}
\def\1{\bm{1}}

\def\rk{{\textnormal{k}}}

\def\vx{{\bm{x}}}
\def\vy{{\bm{y}}}
\def\vz{{\bm{x}}}

\def\mI{{\bm{I}}}

\DeclareMathAlphabet{\mathsfit}{\encodingdefault}{\sfdefault}{m}{sl}
\SetMathAlphabet{\mathsfit}{bold}{\encodingdefault}{\sfdefault}{bx}{n}

\def\gH{{\mathcal{H}}}

\def\gN{{\mathcal{N}}}

\def\gS{{\mathcal{S}}}

\def\gX{{\mathcal{X}}}

\def\gZ{{\mathcal{Z}}}

\def\sP{{\mathbb{P}}}

\newcommand{\E}{\mathbb{E}}

\newcommand{\R}{\mathbb{R}}

\newcommand{\KL}{D_{\mathrm{KL}}}

\newtheorem{theorem}{Theorem}

\newtheorem{lemma}{Lemma}

\newtheorem{definition}{Definition}

\newtheorem{assumption}{Assumption}

\theoremstyle{remark}

\def \trPx {S_X} 
\def \trSx {L_{\tilde{X}}} 
\def \trPz {S_X} 
\def \trSz {L_{\tilde{X}}} 

\def \P {\text{Pr}}
\def \q {q}  
\def \tq {q} 
\def \keta {p_*^{\eta}} 
\def \tketa {p_*^{\eta}} 
\def \hketa {p^{\eta}}

\def \rdm {\mathfrak{R}} 
\def \erdm {\hat{\mathfrak{R}}} 
\def \rk {K_\eta(\gH)} 

\def \KL {\text{KL}}
\def \TV {\text{TV}}

\def \hh {\hat{h}}

\include{defs}
\def\shownotes{1}  
\ifnum\shownotes=1
\newcommand{\authnote}[2]{{$\ll$\textsf{\footnotesize #1: #2}$\gg$}}
\else
\newcommand{\authnote}[2]{}
\fi

\begin{document}

\title{Contrastive learning of strong-mixing\\ continuous-time stochastic processes}

\author{
Bingbin Liu
\and Pradeep Ravikumar
\and Andrej Risteski}

\date{Carnegie Mellon University
\\
\texttt{\{bingbinl,pradeepr,aristesk\}@cs.cmu.edu}}

\maketitle

\begin{abstract}
Contrastive learning is a family of self-supervised methods where a model is trained to solve a classification task constructed from unlabeled data. It has recently emerged as one of the leading learning paradigms in the absence of labels across many different domains (e.g. brain imaging, text, images). However, theoretical understanding of many aspects of training, both statistical and algorithmic, remain fairly elusive.

In this work, we study the setting of time series---more precisely, when we get data from a strong-mixing continuous-time stochastic process. We show that a properly constructed contrastive learning task can be used to estimate the transition kernel for small-to-mid-range intervals in the diffusion case. Moreover, we give sample complexity bounds for solving this task and quantitatively characterize what the value of the contrastive loss implies for distributional closeness of the learned kernel. As a byproduct, we illuminate the appropriate settings for the contrastive distribution, as well as other hyperparameters in this setup. 
\end{abstract}

\section{Introduction}

One of the paradigms of learning from unlabeled data that has seen a lot of recent work in various application domains is ``self-supervised learning''. 
 These methods supervise the training process with information inherent to the data without requiring human annotations, and have been applied across computer vision, natural language processing, reinforcement learning and scientific domains.

Despite the popularity, they are still not very well understood---both on the theoretical and empirical front---often requiring extensive trial and error to find the right pairing of architecture and learning method. In particular, it is often hard to pin down what exactly these methods are trying to learn, and it is even harder to determine what is their statistical and algorithmic complexity.  

The specific family of self-supervised approaches we focus on in this work is \textit{contrastive learning},
which constructs different types of tuples by utilizing certain structures in the data and trains the model to identify the types.
For an example in vision, \cite{simclr} apply two random augmentations (e.g. crops and discolorations) on each training image, and form pairs that are labeled as either positive or negative depending on whether two augmentations are from the same image or not.
In NLP, one of the tasks in \cite{bert,topic} trains the model to predict whether two half-sentences are from the same original sentence.

In this paper, we focus on understanding a natural type of contrastive learning tasks for time series data---a natural structure in NLP \citep{bert,topic}, finance \citep{closed}, and brain imagining research \citep{TCL}
More precisely, we focus on data coming from a discretization of a \emph{diffusion} process---a common modeling assumption in many of these domains---and show that a natural distinguishing task we set up on pairs of samples from the time series approximately learns the transition kernel of the stochastic process.

Note, a diffusion process is a continuous-time stochastic process and we are interested in learning transition kernels for ``mid-range"
time intervals, that is, intervals that are potentially too large for the Euler scheme to be accurate. These transition kernels are not easy to learn in general through standard maximum likelihood methods, as closed-form solutions are complicated~\citep{closed} and often do not exist, and empirical estimations can also be challenging~\citep{milstein2004transition}.
To our knowledge, our work is the first one to use contrastive learning to learn such transition kernels.
Moreover, we provide a statistical complexity analysis---that is, analyzing the number of samples required to learn a good approximation of the transition kernel. This helps quantify certain aspects of contrastive learning
---how should we choose the contrast distribution,
and how a small loss on the contrastive task transfers to closeness of the transition kernel estimate. 


\section{Related Work} 

There is a large body of recent empirical work on self-supervised learning in general, which we won't make an effort to survey in full, as it does not directly relate to our results. 

There have been some recent works on trying to understand theoretically why and when self-supervised learning works. The closest ones in spirit to our work are \cite{tosh} and \cite{TCL}, but there are significant differences with both. \cite{tosh} focus on a data distribution coming from LDA (topic modeling), and characterize the kinds of downstream classification tasks the learned predictor is useful for. \cite{TCL} focus on a time series setting as well but with several differences as highlighted below.

First, they work with a latent-variable model, and show that their method recovers some function of the latent. One example parametrization is an exponential family, and the function of the recovered latent variable depends on the choice of the exponential family.

Second, they assume the data in the time series can be subdivided into ``blocks'', such that the distribution remains the same in each block and is sufficiently different from the others. In practice, it is not clear how to choose these ``blocks'' or how to even verify the assumptions needed on them. We do not need this ``blocking'', but our data needs to come from the stationary distribution of the process. 

Third, they do not provide an analysis on statistical complexity. In particular, important aspects of how various hyperparameters are chosen and affect the quality of the learned predictor---the size of the blocks, the amount of ``difference'' between the blocks---are not clear.

For temporally dependent and stationary data, another related work is by \cite{hyvarinen2017nonlinear}.
The setup is however different: \cite{hyvarinen2017nonlinear} focus on discrete-time data with autocorrelations, whereas we analyze a continuous-time diffusion, leading to different setups and goals for the contrastive task.
Moreover, in contrast to our finite sample analysis, their results describe only the asymptotic behaviors, which can hide certain statistical aspects of the algorithm as discussed earlier.

In the simpler iid setting, a classical precursor paper to this is by \cite{nce}, who apply the contrastive learning approach to learning a distribution from iid samples---by setting up a classification task to distinguish between samples from the target distribution and a simple ``contrast'' distribution. Their analysis is again asymptotic and only provide sample efficiency bounds in the asymptotic limit (i.e. as the number of samples goes to infinity). More recently, such classical approaches have been combined with more modern generative models based approaches to generate better contrast distributions \citep{gao2020flow}, and augmented with intermediate tasks to better handle dissimilar target and contrast distributions~\citep{TRE}.

Finally, other papers on empirical and theoretical properties of contrastive learning that are worth mentioning include~\cite{demystify,goodViews} and ~\cite{arora19,understanding}---these are not directly comparable to what we are doing here, as the data models are quite different, as is the flavor of guarantees they show.
In particular, these papers work with iid data and focus on learning good representations that can perform well on certain supervised tasks,
whereas we use contrastive learning to perform distribution learning, that is, learning the transition kernels.


\section{Main Results}
\label{sec:setup}

This section formally states our results. We will start with specifying the distributional model for the data and the contrastive learning task, and build intuitions on what the task aims to achieve before stating the formal guarantees.
\vspace{-0.3em}

\subsection{Setup} 
We will assume our data comes from continuous time series: namely  $\{\vx_t\}_{t \geq 0} \subset \R^d$, drawn according to a stochastic process called the \emph{Langevin diffusion}\footnote{The results we state can more generally be stated about an \'Ito diffusion, namely a stochastic differential equation of the type  $\displaystyle d\vx_t = - g(\vx_t) dt + \sigma(\vx_t) dW_t, \ \forall t \geq 0$
satisfying similar regularity conditions to ours. We chose the simplest setting for clarity of exposition.}, defined by the stochastic differential equation
\begin{equation}\label{eq:diffusion_def}
    d\vx_t = -\nabla f(\vx_t) dt + \sqrt{2} dW_t, \ \forall t \geq 0,
\end{equation}
for $f: \R^d \rightarrow \R$ a convex function, and $\{W_t\}_{t \geq 0} \subset \R^d$ a Wiener process, i.e. $W_s - W_t \sim \gN(0, (s-t)\mI_d)$, $\forall s$ $\geq t \geq 0$.
For the reader unfamiliar with diffusions, we can think of a diffusion process as the limit of a discrete sequence of noisy gradient updates with a fresh Gaussian noise: as $\eta \rightarrow 0$, the discrete sequence defined by $\vx_{t+1} = \vx_t -\eta\nabla f(\vx_t) + \sqrt{2\eta}\xi_t$ where $\xi_t \sim \gN(0, \mI)$ converges to the continuous time diffusion~\citep{bha78}.
The simplest instantiation of this, when $f$ is quadratic (i.e. $\nabla f$ is linear), gives rise to the Ornstein–Uhlenbeck process, which has broad applications in science and finance modeling.


It is well-known~\citep{bha78} that the stationary distribution of the above process is the distribution $\pi(x) \propto e^{-f(x)}$, under relatively mild regularity conditions on $f$. We will assume that $x_0$ (and hence all subsequent $x_t$) marginally follow $\pi$---i.e. the process is \emph{stationary}. 

We will also need several common assumptions on the $f$ in the generative process.
%
\begin{assumption}[Strong convexity]\label{assump:sc}
$f$ is $\rho$-strongly convex. 
\end{assumption}
\begin{assumption}[Smoothness of $f$]\label{assump:reg}
$f$ is infinitely differentiable, $L_0$-smooth, and $\nabla f$ is $L_1$-smooth.
\footnote{
Recall a function $f$ is $L$ smooth if for any $\vx,\vy$ in the support,
$f(\vx) \leq f(\vy) + \langle \nabla f(\vy), \vx-\vy\rangle + \frac{L}{2}\|\vx-\vy\|_2^2$.}
\end{assumption}
\begin{assumption}[Linear growth]\label{assump:growth}
There exists a positive constant $K < \infty$, such that $\forall \vz \in \gZ$,
$\|\nabla f(\vz)\| \leq K(1 + \|\vz\|)$.
\end{assumption}
\vspace{-0.3em}
Assumption \ref{assump:sc} ensures the least singular value of the Hessian is lower bounded by $\rho$---which ensures that $\int_x e^{-f(x)}$ is finite. Assumption \ref{assump:growth} ensures the existence of a solution to equation \ref{eq:diffusion_def},
and Assumption \ref{assump:reg} ensures the solution of equation \ref{eq:diffusion_def} is unique.\footnote{Milder conditions on $f$ ensure uniqueness/existence of solutions and not essential for our proofs---we assume this for simplicity of exposition.}  
We refer the readers to \cite{closed} for formal justifications of these assumptions.
We denote with $\vz_*$ the minimizer of $f$, and assume $\vz_* = \vec{0}$ for convenience.

Finally, denote $B := \E_{\pi} \|\vx\|$. Note that our assumptions on $f$ guarantee a bounded $B$:
let $Z_\pi := \int_{\vz} \exp\left(-f(\vz)\right)d\vz$ denote the partition function of the stationary distribution $\pi$, and with $\vx_* = \vec{0}$, we have
\begin{equation}
\begin{split}
    &\ B = \frac{1}{Z_\pi} \int_\vz \|\vz\| \exp(-f(\vz)) d\vz 
    \leq \frac{1}{Z_\pi} \int_\vz \|\vz\| \exp\left(- f(\vz_*) - \frac{\|\vz\|^2}{2(1/\rho)}\right) d\vz 
    \\
    =&\ \pi(\vz_*) \E_{\gN(0, \frac{1}{\rho}\mI_d)} \|\vz\|
    \leq \pi(\vz_*)\sqrt{\E_{\gN(0, \frac{1}{\rho}\mI_d)}\|\vx\|^2}
    = \pi(\vz_*)\sqrt{\frac{d}{\rho}}.
\end{split}
\end{equation}

\subsubsection{Contrastive learning task}
\label{sec:task}

We choose the contrastive task to be binary classification on observations from the diffusion defined in equation \ref{eq:diffusion_def}.
For $\eta = O_{L_1, L_2, \rho}(1)$---i.e. any $\eta$ sufficiently small as a function of the regularity parameters of $f$---we will consider the observations at integer multiples of $\eta$, namely $\trSx := \{\tilde{\vx}_{i\eta}\} \subset \R^n$, and let $T > 0$ be length of the (continuous-time) sequence covered by these observations. Suppose the number of observations in $\trSx$ is $2m$ where $2m = \floor{T/\eta}$. 

The binary classification task is defined on a sequence of pairs of points denoted as $\trPx := \{(\vx_{2i\eta},\vx_{2i\eta}')\}_{i=0}^{m-1}$, where $\vx_{2i\eta} = \tilde{\vx}_{2i\eta}$, and $\vx_{2i\eta}'$ is chosen in one of the two ways:
\vspace{-0.8em}
\begin{itemize}[leftmargin=*]
\item With probability $1/2$, we let $\vx_{2i\eta}' = \tilde{\vx}_{(2i+1)\eta}$ and output $(\vx_{2i\eta}, \vx_{2i\eta}')$ with label $1$. We call these \emph{positive} pairs. 
\item With probability $1/2$, we sample $\vx_{2i\eta}' \sim \q$ for some \emph{contrast} proposal distribution $\q$ and output $(\vx_{2i\eta}, \vx_{2i\eta}')$ with label $0$. (We will specify the restrictions on $\q$ momentarily.) We call these \emph{negative} pairs. 
\end{itemize}
Intuitively, the task asks the model to distinguish the noise distribution $q$ from the $\eta$-time transition kernel of the process $\tketa: \R^{d} \times \R^{d} \rightarrow \R_{\geq 0}$, which is defined as
\begin{equation}
    \tketa(\vz, \vz') := \P(\vz_{(t+1)\eta}=\vz'|\vz_{t\eta}=\vz).
\end{equation}
What we need to assume on the contrast distribution $\q$ is that it is sufficiently close to $\keta$ (algorithmically, it also needs to have a pdf that is efficient to evaluate).
Specifically, define a constant $c_q \geq 1$, such that the ratio between $\keta$ and the proposal distribution $q$ is bounded as 
\begin{equation}\label{eq:cq_def}
    \frac{1}{c_q} \leq \frac{\keta(\vz, \vz')}{q(\vz')} \leq c_q,\ \forall \vz, \vz'.
\end{equation}
We will show later that a smaller $c_q$ is more preferable, which amounts to choosing a proposal distribution $\q$ that closely tracks the data distribution.
This is consistent with observations in previous works on noise contrastive learning (NCE)
that a closer $\q$ makes the contrastive task harder and hence tends to work better in practice~\citep{nce,gao2020flow}.
Formally, a larger $c_q$ will give a looser bound on the KL divergence between $\hketa$ and $\keta$,
as we will see in theorem \ref{thm:kl} and its proof.

The model we use for the supervised task is denoted as $h: \R^{2n} \rightarrow \R$, which takes in a $(\vx, \vx')$ pair and predicts the probability of the pair being positive. We assume the output of $h$ to be bounded in $[0,1]$.\footnote{This can easily be enforced, for example, by having a sigmoid layer at the end of a neural network.}
We denote the function class $h$ belongs to as $\gH$, and train $h$ with the $\ell_2$ loss:
\begin{equation}\label{eq:loss}
\begin{split}
    \ell\left(h, \{(\vx,\vx'), y\}\right)
    &= \left(h(\vx,\vx') - y\right)^2
    \\
\end{split}
\end{equation}
Let the \textit{empirical risk} $\hat{R}$ of a model $h$ and loss $\ell$ associated with a training set $\trPx$ be defined as usual, 
and taking the expectation over $\trPx$ gives the \textit{population risk} $R$:
\begin{equation}\label{eq:ER_def}
\begin{split}
    \hat{R}_{\trPx}(\ell \circ h) &:= \frac{1}{|\trPx|}\sum_{i=1}^{|\trPx|} \ell(h, \{(\vx_{2i\eta}, \vx_{2i\eta}'), y\}),
    \\
    R(\ell \circ h) &:= \E_{\trPx} \hat{R}_{\trPx}(\ell \circ h).
\end{split}
\end{equation}
The \textit{generalization gap} is defined as the maximum difference between the above two in the class of classifiers that we consider:
\begin{equation}\label{eq:gen_gap_def}
    \Phi(\trPx) := \sup_{h\in \gH}\left[\left|\hat{R}_{\trPx}(\ell \circ h) - R(\ell \circ h)\right|\right]
\end{equation}

By way of remarks: the $l_2$ loss is chosen since it is bounded, Lipschitz and strongly-convex, which makes the generalization bound calculations more manageable. It would be interesting to also provide bounds for cross-entropy or other losses. 

We also need concepts capturing the complexity of the function class: the \emph{empirical Rademacher complexity} $\erdm$ of a function class $\gH$ is defined with a given dataset $S$ of size $m$ as
\begin{equation}
    \erdm_S(\gH) := \frac{1}{m}\E_{\varepsilon} \left[\sup_{h \in \gH}\Big|\sum_{i \in [m]} \varepsilon_i h(x_i) \Big| | S = (x_1, ..., x_{m})\right]
\end{equation}
The \textit{Rademacher complexity} is defined by taking the expectation over the dataset $S$ as
\begin{equation}
    \rdm_m(\gH) := \E_{S:|S|=m} \erdm_S(\gH).
\end{equation}

\begin{assumption}[Rademacher Complexity]
We will assume the Rademacher complexity of $\gH$ satisfies $\rdm_{m}(\gH) = O(\rk\sqrt{\log m / m})$, where $\rk$ depends on both $\gH$ and the task setup $\eta$. 
\end{assumption}
The expression for $\rdm_m(\gH)$ is common in standard generalization bounds. For instance, such dependency is achieved when the square root of the VC dimension of $\gH$ is bounded by $\rk$~\citep{foundation12}.


\subsection{Characterizing the Optimum given Infinite Data} 

To gain intuition on what the contrastive task does, we first characterize the optimum of the contrastive learning objective in the limit of infinite data. We note that similar analyses have appeared in other works on variants of contrastive learning, e.g. \cite{TCL,topic,tosh}. 
We show: 
\begin{lemma}[Population optimum]
\label{lem:optimum}
The optimum of the contrastive learning objective 
$$\underset{h}{\arg\min} \E_{((x,x'),y)} \ell\left(h, \{(\vx,\vx'), y\}\right) $$
satisfies 
\vspace{-0.8em}
$$ h^*(\vx, \vx') = \frac{\tketa(\vz, \vz')}{\tq(\vz') + \tketa(\vz, \vz')}$$
\label{l:opt}
\vspace{-0.8em}
\end{lemma}
\vspace{-0.8em}
\begin{proof}
The proof proceeds by expanding the expectation in question, and a variance-bias like calculation. Namely, for a fixed $(x,x')$, taking the expectation over $y$ gives:
\vspace{-0.6em}
\begin{equation}\label{eq:loss_train}
\begin{split}
    \E_y \ell\left(h, \{(\vx,\vx'), y\}\right)
    = \left(h(\vx,\vx') - \P(y=1|\vx,\vx')\right)^2 
    + \P(y=1|\vx,\vx')\left(1 - \P(y=1|\vx,\vx')\right)
    \\
\end{split}
\end{equation}
The last term of \eqref{eq:loss_train} does not depend on $h$, so the minimum is achieved when $h^*(\vx,\vx') = \P(y=1|\vx,\vx')$. 
Expanding $\P(y=1|\vx,\vx')$
by the Bayes rule, we get
\begin{equation}\label{eq:optim}
\begin{split}
    &\ h^*(\vx, \vx') = \P(y=1|\vx,\vx')
    \\
    =&\ \frac{\P(\vx, \vx'|y=1)\P(y=1)}{\P(\vx,\vx'|y=0)\P(y=0) + \P(\vx, \vx'|y=1)\P(y=1)}
    \\
    =&\ \frac{\P(\vx, \vx'|y=1)}{\P(\vx,\vx'|y=0) + \P(\vx, \vx'|y=1)}
    \\
    =&\ \frac{\pi(\vz)\tketa(\vz, \vz')}{\pi(\vz)\tq(\vz') + \pi(\vz)\tketa(\vz, \vz')}
    = \frac{\tketa(\vz, \vz')}{\tq(\vz') + \tketa(\vz, \vz')}.
\end{split}
\end{equation}
\end{proof}
Note that the above proof uses essentially nothing about $\tq$ other than that it is known: this is why population level analyses of contrastive objectives (e.g. like \cite{TCL}) may fail to capture many non-asymptotic aspects of the contrastive task.

\subsection{Statement of main results}\label{sec:main_results}
We claim that a low loss on the contrastive task implies closeness in a learned $\eta$-time transition kernel and the ground truth one.
We will state the main results here and defer the proofs to Section \ref{sec:proof}.

\paragraph{Sample complexity bounds} We first present the sample complexity for controlling the generalization gap defined in equation \ref{eq:gen_gap_def}:
\begin{theorem}\label{thm:opt_gen_bdd}
If $T = \Omega\left(\frac{B^2\rk^3}{\delta^2\Delta_{gen}^3}\left(\log\frac{1}{\delta}\right)^{\frac{3}{2}}\right)$, then with probability $1-\delta$, the generalization gap is bounded by $\Delta_{gen}$. 
\end{theorem}

 Note that the dependency of $T$ on $\eta$ comes only through the complexity measure $\rk$. The reason there isn't additional dependence on $\eta$ (e.g. the reader might imagine the number of ``samples'' effectively depends on $T/\eta$) is that though decreasing $\eta$ gives more samples, the samples will be more dependent and hence less useful for generalization. The proof in section \ref{sec:proof_gen} will formally justify this intuition.

\paragraph{Distribution estimator from classifier for contrastive task} Second, we show how to prove guarantees on an estimator for the transition kernel, given a classifier with a small contrastive task loss. 

In light of Lemma \ref{lem:optimum}, given a classifier $h$, define the transition kernel implied by $h$ as
\begin{equation}
    \hketa(\vz, \vz') := \frac{h(\vz,\vz')\tq(\vz')}{1 - h(\vz,\vz')}.
\end{equation}

We wish to show that if $h$ achieves a small loss, the $p^{\eta}$ defined above is in fact close to $p^*$ in some distributional sense. 

We will show two types of guarantees, one under the assumption that $p^{\eta}$ is somewhat close to $p^*$, and one for arbitrary $p^{\eta}$. 
In the first case, we will in fact show that a small loss implies that the learned $p^{\eta}$ is close to $p^*$ in a KL divergence sense (more precisely, $\E_{x} KL(p^{\eta}_*(\cdot | x) \| p^{\eta}(\cdot | x))$ is small); in the second case, we will show that $\E_{x,x'} |p^{\eta}_*(x, x') - p^{\eta}(x, x')|$ is small. 

The reason we can extract a stronger result in the first case is that we can leverage the strong convexity of the contrastive loss near the global optimum in an appropriate sense. Intuitively, in a strongly convex loss, a small loss implies closeness of the parameter to the global optimum. Such a property will not hold globally, as the loss may be arbitrarily non-convex as a function of $p^{\eta}$. Still, we will be able to extract a weaker guarantee (and with a less standard notion of distance). 

\paragraph{Guarantees for $p^{\eta}$ close to $p_*^{\eta}$}

Let constants $\Delta_{\min}$, $\Delta_{\max}$ be defined such that
\begin{equation}
    \frac{\hketa(\vz, \vz')}{\keta(\vz, \vz')} \in [\Delta_{\min}, \Delta_{\max}],\ \forall \vz, \vz'.
\end{equation}
and note that $0 < \Delta_{\min} \leq 1 \leq \Delta_{\max}$.

$\Delta_{\min},\Delta_{\max}$ can be considered as a notion of closeness between $\keta$ and $\hketa$.
When $\Delta_{\min}, \Delta_{\max}$ are close to 1, that is, when $\hketa$ lies in a small neighborhood of $\keta$, we can show the contrastive loss is locally strongly convex with respect to the KL divergence. This allows us to relate the loss to the KL divergence between $\keta$ and $\hketa$.
Formally, we state the following result:
\begin{theorem}\label{thm:kl}
Suppose assumption \ref{assump:sc}-\ref{assump:growth} are satisfied and that $\Delta_{\max} \leq \frac{7}{6}$.
Suppose the training error of $h$ is $\epsilon_{tr} + \epsilon_\star$, where $\epsilon_\star := \E_{\{(\vx,\vx'),y\}}\left(\frac{\tketa(\vx,\vx')}{\tketa(\vx,\vx')+\tq(\vx')} - y\right)^2$ is the optimal error achieved by $\keta$.
If the generalization gap is bounded by $\epsilon_{tr}$,
then the average $\KL$ divergence between the ground truth and learned transition kernel is bounded by the contrastive loss as
\begin{equation}
    \E_{\vz\sim\pi}\KL\left(\keta(\cdot |\vz) \| \hketa(\cdot|\vz)\right) 
    \leq \frac{2(1+c_q)^5\epsilon_{tr}}{\Delta_{\min}^2}.
\end{equation}
\end{theorem}
Recall that $c_q$ (defined in equation \ref{eq:cq_def}) represents how close the contrast distribution $\q$ is to $\keta$. Theorem \ref{thm:kl} hence explains why a closer $\q$ is more preferable, as has been suggested by empirical evidence~\citep{gao2020flow}.
In addition, as mentioned earlier, the bound $\Delta_{\max} \leq \frac{7}{6}$ is required to reason about the convexity of the contrastive loss in the neighborhood of $\keta$. Globally, the loss need not be convex, so it is entirely possible for a $\hketa$ faraway from $\keta$ in the KL sense to have a small contrastive loss.
Nevertheless, we can prove something weaker in this case. 

\paragraph{Guarantees for arbitrary $p^{\eta}$}
In the case of an arbitrary $p^{\eta}$, we prove the following bound on the closeness to $p^{\eta}_*$: 
\begin{theorem}\label{thm:orig}
Under assumptions \ref{assump:sc}-\ref{assump:growth}, and let $\epsilon_*, T, \hketa$ be the same as in Theorem \ref{thm:kl}.
Let $\eta = O_{\rho, L_0, L_1}(1)$, then $\hketa$ satisfies
\begin{equation}\label{e:main}
\begin{split}
    &\ \underset{\vz\sim \pi,\vz'\sim \tketa(\cdot | \vz)}{\E}\big|\hketa(\vz,\vz')-\tketa(\vz,\vz') \big|
    \leq
    \sqrt{2\epsilon_{tr}}  
    O\left(\sqrt{\pi(\vz_*)}\left(\frac{1}{\rho\eta^2}\right)^{d/4}\right)
\end{split}
\end{equation}
\end{theorem}
\vspace{0.2em}

We make two remarks about the theorem \ref{thm:orig}.
First, the value of $\eta$ cannot be too large for the RHS of \eqref{e:main} to obtain (i.e. $\eta = O_{\rho, L_0, L_1}(1)$). 
Analyzing merely the optimum of the contrastive objective would not reveal this.

Moreover, though the exponential dependency in $\eta$ may appear pessimistic, it is in fact the right one. 
A closer inspection of the left hand side of equation \ref{e:main} shows that its scaling in $\eta$ is also $(1/\eta)^{d/2}$ (by Lemma~\ref{lem:kernel_bdd})---so the only ``extra'' exponential factors are the $\eta$-independent exponential terms. It is not clear if this can be removed or is essential---or if possibly other losses can remove this kind of dependence. 



\section{Generalization Machinery for Non-iid Data}
\label{sec:tech}


At the core of our analysis is a set of tools for non-iid data, which we first build up before discussing the proof.
We will use generalization results for data coming from \emph{strong mixing} stochastic processes: namely, the samples are not independent; but, intuitively, after a short amount of time, the samples are ``almost independent''. Precisely, we use the notion of $\beta$-mixing: 

\begin{definition}[$\beta$-mixing]
For a stationary Markov process,
the $\beta$-mixing coefficient is defined as the average TV distance between the distribution after running the process for $t$ time with a given starting point, and the stationary distribution $\pi$:
\begin{equation}
    \beta(t) = \E_{\vx}\ \TV\left(P^t(\cdot | x_0 = \vx), \pi\right)
\end{equation}
A process is said to be $\beta$-mixing if $\ \lim_{t \rightarrow \infty} \beta(t) = 0$.

The $\beta$-mixing coefficient of a discrete-time sequence is defined similarly, with the conditional distribution defined between points in the sequence.
\end{definition}
We note that $\beta$-mixing can be defined more generally on processes that may are not necessarily stationary or Markov. The above definition is the cleanest version that suffices for our setting.

The reason this will be useful for us is that when our data is $\beta$-mixing, we will be able to use generalization bounds similar to those we have for iid data. More precisely, we will leverage the following result by \cite{mohri09}, which when applied to our setting becomes:
\begin{lemma}[Rademacher complexity bound, \cite{mohri09}, Theorem 1]
\label{lem:gen_bdd}
Let $\gS_X$ form a $\beta$-mixing sequence with stationary distribution $\pi$. Then, for some $\delta \in (0,1)$, for every $\mu$ such that $\delta > 2(\mu-1)\beta(T/2\mu)$,
with probability at least $1 - \delta$,
the generalization gap $\Phi(\gS_X)$ is bounded by
\vspace{-0.4em}
\begin{equation}\label{eq:gen_bdd_def}
\begin{split}
    \Phi(\trPx) &:= \sup_{h \in \gH} \left[R(\ell \circ h) - \hat{R}_{\trPx}(\ell \circ h)\right]
    \leq \left\{\rdm_{\mu}(\gH) + \sqrt{\frac{\log\left(2/\left(\delta - \Delta^{\mu}_{appr}\right)\right)}{2\mu}}\right\}
\end{split}
\end{equation}
where $\Delta^{\mu}_{appr} := 2(\mu - 1)\beta_{S_X}(T/2\mu)$, and $\rdm_\mu(\gH)$ is the Rademacher complexity of $\gH$ over samples of size $\mu$ drawn iid from $\pi$.
\end{lemma}

The result is proved using a technique called \textit{blocking} from \cite{yu94}. The idea is to divide a dependent sequence of samples into $2\mu$ blocks of consecutive points, such that when the block size $\frac{T}{2\mu}$ is sufficiently large, every other block would be approximately independent because of the fast mixing. The generalization analysis can hence be divided into two steps, one for applying standard generalization bound on i.i.d. data (i.e. the blocks), and the other for bounding the approximation error of treating dependent blocks as independent ones. The term $\Delta_{appr}^\mu$ is a consequence of the derivation in \cite{yu94} and accounts for errors of approximating non-iid data with iid ones.

We will proceed by first showing fast mixing, then applying the generalization bounds above. 


\subsection{Proving \texorpdfstring{$\beta$}{beta}-mixing}
We will first show $\beta$-mixing of the sequence $\trPx$ of pairs $(\vx,\vx')$ as constructed in Section \ref{sec:setup}; that is,
by choosing $\vx$ from the diffusion process, and then choosing $\vx'$ to be $\eta$-time after in the process or from a proposal distribution with equal probability.
Intuitively, this would suggest that once two points are sufficiently apart, they will be approximately independent, on which standard generalization bounds apply.
Formally, we have the following result:
\begin{lemma}\label{lem:beta}
The $\beta$-mixing coefficients for the sequence $\trPz$ defined in Section \ref{sec:setup} is
$\beta_{\trPz}(t) = O\left(\frac{B}{\sqrt{t}}\right)$.
\end{lemma}
\begin{proof}
We will prove this by showing the sequence of pairs $\trPz$ shares the same $\beta$ coefficients as the sequence of points $\trSz$ (Lemma \ref{lem:beta_pair}). 
Then, since $\trSz$ is itself $\beta$-mixing (Lemma \ref{lem:beta_ito}), the claim follows.
\end{proof}

Having the same $\beta$ coefficients between $\trPx$ an $\trSz$ makes intuitive sense, since the sequence of pairs can be considered as a mixture of a dependent sequence and an independent sequence, and adding the independent one should not worsen the mixing coefficient.
\begin{lemma}\label{lem:beta_pair}
$\beta_{\trPz}(t) = \beta_{\trSz}(t)$.
\end{lemma}
\begin{proof}
First note that $\trPz$ is Markov and stationary, since the temporal dependency only comes from the first elements in the pairs, which are points in $\trSz$ that is itself Markov and stationary:
\begin{equation}
\begin{split}
    &\P\left((\vx_{2(i+1)\eta}, \vx_{2(i+1)\eta}')| (\vx_{0}, \vx_{0}'), ..., (\vx_{2i\eta}, \vx_{2i\eta}')\right)
    \\
    =&\ \P\left(\vx_{2(i+1)\eta}'|\vx_{2(i+1)\eta}\right) \P\left(\vx_{2(i+1)\eta}| \vx_{0}, ..., \vx_{2i\eta}\right)
    \\
    =&\ \P(\vx_{2(i+1)\eta}'|\vx_{2(i+1)\eta}) \P\big(\tilde{\vx}_{2(i+1)\eta}| \tilde{\vx}_{2i\eta}\big)
    \\
    =&\ \P(\vx_{2(i+1)\eta}'|\vx_{2(i+1)\eta}) \P\big(\vx_{2(i+1)\eta}| \vx_{2i\eta}\big)
\end{split}
\end{equation}
The mixing coefficient of $\trPx$ can then be calculated explicitly, leading to $\beta_{\trPz}(2i\eta) = \beta_{\trSz}(2i\eta).$
The details are deferred to appendix \ref{appendix:lem_mix_coeff}.
\end{proof}
\vspace{-1em}
Next, we bound the TV distance, as a function of $t$, between the stationary distribution $\pi$ and the distribution after running the diffusion for time $t$ given any starting point:
\begin{lemma}[\cite{bubeck}, Proposition 4]
\label{lem:beta_ito}
Let $B := \E_{\pi} \|x\|$. For any $t > 0$, $\forall \vz \in \gX$,
$$\TV(\sP(\vz_t|\vz_0 = \vz), \pi) \leq \frac{B}{\sqrt{2\pi t}},$$
where $\sP(\vz_t | \vz_0 = \vz)$ denotes the distribution after running the diffusion for time $t$ conditioned on being at $\vz$ at time 0.
\end{lemma}
\vspace{-0.5em}
With the definition of $\beta$-mixing, Lemma \ref{lem:beta_ito} shows that $\trSz$ itself is $\beta$-mixing, as long as $B < \infty$, $\beta_{\trSz}(t\eta) = \TV(\sP(\vz_{t\eta}|\vz_0 = \vz), \pi) = O(\frac{1}{{\sqrt{t\eta}}}) \rightarrow 0$ as $t \rightarrow \infty$. 

\section{Proofs of Main Results}\label{sec:proof}
We are now ready to prove the main results in section \ref{sec:main_results}.
We will start with the finite sample generalization bound, and map the loss on the contrastive task to the KL divergence between the learned and true transition kernels, assuming the former lies in a neighborhood of the latter. We will finish with the proof for theorem \ref{thm:orig} where the closeness assumption is lifted.
\subsection{Proof of the Generalization Bound}
\label{sec:proof_gen}
Let's first prove the sample complexity bound for generalization, where we use results in \cite{mohri09} to choose the optimal $\mu$ to bound the generalization gap.
\begin{proof}[Proof of Theorem \ref{thm:opt_gen_bdd}]
Following notations in lemma \ref{lem:gen_bdd}, let $\mu$ denote the number of ``effective" training samples.
Substituting in the choice of $T = \Omega\left(\frac{B^2\rk^3}{\delta^2\Delta_{gen}^3}\left(\log\frac{1}{\delta}\right)^{\frac{3}{2}}\right)$, we have $\Delta_{appr}^\mu$ $= O\left(\frac{B}{\sqrt{T}}\mu^{\frac{3}{2}}\right)$ $\leq \delta$.
Recall the empirical Rademacher complexity is $\rdm_{\mu} = O\left(\rk\sqrt{\log\mu/\mu}\right)$.
Then, choosing $\mu$ such that $\mu = \Omega\left(\frac{\rk\sqrt{\log(1/(\delta - \Delta_{appr}))}}{\Delta_{gen}^2}\right)$, it can be checked that the following is satisfied:
\begin{equation}
     C\sqrt{\frac{1}{\mu}} \left(\rk\sqrt{\log \mu} + \sqrt{-\log\left(\delta - \Delta_{appr}\right)} \right) \leq \Delta_{gen}
\end{equation}
The calculation details can be found in appendix \ref{appendix:thm_sample_cmplx}.

\end{proof}



\subsection{Proof of Theorem \ref{thm:kl}}
\label{sec:proof_thm1}

Theorem \ref{thm:kl} states that when $\hketa$ is close to $\keta$ and the population contrastive loss is not much worse than the optimal value, the KL divergence between $\keta$ and $\hketa$ is also small.
At a high level, with $\hketa$ close to $\keta$, we can do a ``multiplicative'' Taylor expansion of $\hketa$ around $\keta$. Then, it can be shown that the second derivative is strictly positive with a proper choice of $\Delta_{\max}$. This is similar in spirit to the notion of strong convexity with respect to KL, from the difference in losses.
\vspace{-0.8em}
\begin{proof}[Proof of Theorem \ref{thm:kl}]
Recall that in section \ref{sec:main_results} we defined constants $\Delta_{\min}$, $ \Delta_{\max}$ such that $\forall \vx,\vx'$, 
$\frac{\hketa(\vz, \vz')}{\keta(\vz, \vz')} \in [\Delta_{\min}, \Delta_{\max}]$.
We can equivalently write this relation as $\hketa = \keta(1 + \delta)$ with $\delta \in [\Delta_{\min}-1, \Delta_{\max}-1]$.
By the mean value theorem, $\exists\ \xi \in [\Delta_{\min}-1, \Delta_{\max}-1]$, such that
\begin{equation}\label{eq:kl_delta}
\begin{split}
    &\E_{\vz\sim\pi,\vz'\sim\keta}\KL(\keta\| \hketa) = \E_{\pi} \int \keta \log\frac{\keta}{\keta(1+ \delta)}
    \\
    =& - \E_{\pi}\int \keta \log(1+\delta)
    = -\E_{\pi} \int \keta \left(\delta -\int_{s=0}^\delta \frac{1}{2(1+s)^2}s\ ds\right)
    \\
    \overset{(i)}{=}&\ -\E_{\pi}\int \keta (\delta - \frac{1}{2}\frac{1}{(1+\xi)^2} \delta^2)
    \\
    \overset{(ii)}{\leq}& \frac{1}{2\Delta_{\min}^2}\E_{\vz\sim\pi,\vz'\sim\keta}\delta^2.
    \\
\end{split}
\end{equation}
where $(i)$ applies the mean value theorem to the second order Taylor expansion around 0, and $(ii)$ uses $\int_x p_*^\eta \delta = 0$ since $p^\eta, p_*^\eta$ both integrates to 1.

Rewriting the gap between the population loss between $\hketa(\vz,\vz')$ and $\keta(\vz,\vz')$ as a function of $\delta(\vz, \vz')$, we get:
\begin{equation}
\begin{split}
    r(\delta) &:= \left(\frac{\keta \q \delta}{(\keta(1+\delta)+\q)(\keta+\q)}\right)^2
\end{split}
\end{equation}
where the dependence on $\vz,\vz'$ is omitted for clarity.

We would like to lower bound $r''(\delta)$.
By the mean value theorem,
for any $\vz,\vz'$,
$\exists\ \xi' \in [\Delta_{\min}-1, \Delta_{\max}-1]$, 
\begin{equation}\label{eq:r_2nd_bdd}
\begin{split}
    r''(\delta) &= r''(0) + r'''(0) \delta + \frac{1}{2} r''''(\xi')\delta^2
    \\
    &\geq \frac{2(\keta)^2\q^2}{(\keta + \q)^5}\left((7 - 6\Delta_{\max})\keta + \q\right)
    + \frac{12(\keta)^5\q^2}{\keta + \q}
    \cdot \frac{5\keta + 3\q - 2\keta\Delta_{\max}}{(\Delta_{\max}\keta+\q)^6}
    \\
\end{split}
\end{equation}
The Taylor series converges when $|\delta| \leq \frac{\keta+\q}{\keta} = 1 +\frac{\q}{\keta}$, which always holds under our assumption on $\delta$.
Moreover, we require $\Delta_{\max} \leq \frac{7}{6}$ to ensure $(7 - 6\Delta_{\max})\keta + q > 0$.
Then,
\begin{equation}
\begin{split}
    &\frac{5\keta + 3\q - 2\keta\Delta_{\max}}{(\Delta_{\max}\keta+\q)^6}
    \geq \frac{1}{3} \frac{8\keta + 9\q}{(\frac{1}{6})^6(7p + 6q)^6}
    \geq \frac{6^6}{3}\frac{8\keta + 8\q}{(7\keta + 7\q)^6}
    = \frac{8\cdot 6^6}{3\cdot 7^6}\frac{1}{(\keta + \q)^5}
    \geq \frac{1}{(\keta + \q)^5}.
    \\
\end{split}
\end{equation}
Substituting this back to equation \ref{eq:r_2nd_bdd} gives
\begin{equation}
\begin{split}
    r''(\delta) &\geq \frac{2(\keta)^2\q^3}{(\keta + \q)^5}
    + \frac{24(\keta)^5\q^2}{(\keta + \q)^6}
    \geq 2\cdot \frac{1}{(1 + \frac{\q}{\keta})^2}\cdot \frac{1}{(1 + \frac{\keta}{\q})^3}
    \geq \frac{2}{(1+c_q)^5}.
\end{split}
\end{equation}
We can then derive an upper bound on $\E\delta^2$. Recall that in theorem \ref{thm:kl} the gap between the population loss of the learned $h$ and that of $h_*$ is set to be $2\epsilon_{tr}$:
\begin{equation}
\begin{split}
    2\epsilon_{tr} &= \E_{\vz\sim\pi, \vz' \sim \frac{\keta+\q}{2}} \big[r(\delta(\vz,\vz')) - r(0) \big]
    \\
    &= \E_{\vz\sim\pi, \vz' \sim \frac{\keta+\q}{2}} \int_{0}^{\delta} (\delta - t)r''(t)dt
    \geq \frac{r_{\min}}{2}\E_{\vz\sim\pi, \vz' \sim \frac{\keta+\q}{2}} \delta^2
    \geq \frac{r_{\min}}{4}\E_{\vz\sim\pi, \vz' \sim \keta}\delta^2
\end{split}
\end{equation}
where we denote $r_{\min} := \frac{2}{(1+c_q)^5}$.

This means $\E_{\vz\sim\pi, \vz' \sim \keta} \delta^2 \leq \frac{8\epsilon_{tr}}{r_{\min}}$.
Together with equation \ref{eq:kl_delta},
we can bound the average KL as
\begin{equation}
\begin{split}
    &\E_{\vz\sim\pi}\KL\left(\keta(\cdot |\vz) \| \hketa(\cdot|\vz)\right) \leq \frac{1}{2\Delta_{\min}^2} \E_{\pi,\keta}\delta^2
    \leq \frac{4\epsilon_{tr}}{r_{\min}\Delta_{\min}^2}
    = \frac{2(1+c_q)^5\epsilon_{tr}}{\Delta_{\min}^2}.
\end{split}
\end{equation}
\vspace{-0.8em}
\end{proof}
\subsection{Proof of Theorem \ref{thm:orig}}
\label{sec:proof_thm2}
Up to this point, we have reasoned about the generalization gap and the relation between the loss and distributional closeness when $\hketa$ is in the proximity of $\keta$.
There is one last piece missing: we need to characterize what the value of the loss implies for $\hketa$ when its relation to $\keta$ is unknown.

This is not an obvious task because the loss guarantees that the squared difference in equation \ref{eq:loss_train} is small \textit{on average} over $\vx,\vx'$ according to our data distribution.
This does not necessarily imply the squared difference in its numerator,
\footnote{Recall that the squared term in equation \ref{eq:loss_train} can be expanded as
$\left(h(\vx,\vx') - \P(y=1|\vx,\vx')\right)^2 = \left(\frac{p}{p+q} - \frac{p_*}{p_*+q}\right)^2 = \frac{(p - p_*)^2q^2}{(p+q)^2(p_* + q)^2}$.}
i.e. $\left(\hketa(\vz, \vz') - \tketa(\vz, \vz')\right)^2$,
is small.
For example, if $\tq(x') \ll \min\{\hketa(\vz,\vz'), \tketa(\vz, \vz')\}$, then the above difference would be small regardless of the values of $\hketa(\vz, \vz'), \tketa(\vz, \vz')$. 

We will leverage the following estimates on the transition kernel of the Langevin diffusion: 
\begin{lemma}[\cite{gobet2002lan}, Proposition 1.2]
\label{lem:kernel_bdd}
Under assumptions \ref{assump:sc}-\ref{assump:growth},
$\exists\ c,C > 1$, such that 
\begin{equation}
\begin{split}
    \tketa(\vz, \vz') &\geq \frac{1}{c} \frac{1}{\eta^{d/2}} e^{-C\frac{\|\vz-\vz'\|^2}{\eta}}e^{-C \eta \|\vz\|^2},
    \\
    \tketa(\vz, \vz') &\leq c \frac{1}{\eta^{d/2}} e^{-\frac{1}{C}\frac{\|\vz-\vz'\|^2}{\eta}}e^{C \eta \|\vz\|^2}.
\end{split}
\end{equation}
\end{lemma} 
\vspace{-0.8em}
The theorem in \cite{gobet2002lan} holds actually in a substantially more general setting than ours: it only requires that the drift (in our setting $\nabla f$) and diffusion coefficient are in $C^{1+\gamma}, \gamma > 0$. 

With this result in mind, as well as the previous lemmas, we are ready to prove theorem \ref{thm:orig}:
%
\vspace{-0.8em}
\begin{proof}[Proof of Theorem \ref{thm:orig}]
Recall that the optimal solution of the contrastive task satisfies $h^*=\frac{\tketa(\vz, \vz')}{\tq(\vz') + \tketa(\vz, \vz')}$ by Lemma \ref{l:opt}, and that the population contrastive loss is no more than $2\epsilon_{tr}$ over the optimal $\epsilon_*$ achieved by $h^*$.
%
This gives
\begin{equation}
\begin{split}
    2\epsilon_{tr}
    \geq& \E_{\vz \sim \pi} \E_{\vz' \sim \frac{1}{2}(\tketa(\vx,\cdot) + \tq)}  \ell (h)
    \geq \frac{1}{2}\E_{\vz \sim \pi} \E_{\vz' \sim \tketa(\vx,\cdot)} \ell(\hh)
    \\
    =&\ \frac{1}{2}\E_{\vz \sim \pi} \E_{\vz' \sim \tketa(\vx,\cdot)} \left(\frac{\q (\hketa - \keta)}{(\q + \hketa)(\q + \keta)}\right)^2
\end{split}
\end{equation}
We now use the above loss bound to upper bound $\E_{\vz \sim \pi, \vz' \sim \tketa(\vx,\cdot)} |\hketa(\vx,\vx') - \tketa(\vx,\vx')|$. 
For notational convenience, we will drop $\vx,\vx'$ when it is clear from the context. 

Define a function $r_q$ with $r_q(p) = \frac{p}{p+q}$. 
The population risk can now be written as $\E_{\vz,\vz'}\left(r_q(\hketa) - r_q(\tketa)\right)^2$.
Note that $r_q'(p) = \frac{q}{(p+q)^2}$ and $r_q$ is concave in $p$, hence
\begin{equation}\label{eq:concave}
    r_q'\left(\max\{\hketa, \tketa\}\right) \cdot \left(\hketa - \tketa\right) \leq r_q(\hketa) - r_q(\tketa)
\end{equation}

Using equation \ref{eq:concave} and Cauchy-Schwarz, we have
\begin{equation}\label{eq:loss_ker}
\begin{split}
    &\E_{\vz \sim \pi, \vz' \sim \tketa(\vx,\cdot)}|\hketa - \tketa|
    = \E_{\vz \sim \pi, \vz' \sim \tketa(\vx,\cdot)} \left[\frac{|\hketa - \tketa| r_q'\left(\max\{\hketa, \tketa\}\right)}{r_q'\left(\max\{\hketa, \tketa\}\right) }\right]
    \\
    \leq&\ 
    \E_{\vz \sim \pi, \vz' \sim \tketa(\vx,\cdot)}\left[ \left(r_q(\hketa) - r_q(\tketa)\right) \cdot \frac{1}{r_q'\left(\max\{\hketa, \tketa\}\right) }\right]
    \\
    \leq&\ \sqrt{\E_{\vz \sim \pi, \vz' \sim \tketa(\vx,\cdot)}\left(r_q(\hketa) - r_q(\tketa)\right)^2}
    \cdot \sqrt{\E_{\vz \sim \pi, \vz' \sim \tketa(\vx,\cdot)}\left(\max\{\hketa, \tketa\} + q\right)^4/q^2}
\end{split}
\end{equation}
where the first term is the population risk on the contrastive task, which is bounded by $2\epsilon_{tr}$.

We will proceed to bound the second term.
Since $\hketa, \tketa$ are both assumed to satisfy Assumption \ref{assump:sc}-\ref{assump:growth},
lemma \ref{lem:kernel_bdd} allows us to bound the quantity of interest in equation \ref{eq:loss_ker}: let $c_*,C_*$ and $\hat{c}, \hat{C}$ be the constants in lemma \ref{lem:kernel_bdd} for $\tketa$ and $\hketa$ respectively.
Denote $C_u = \max\{C_*, \hat{C}\}$, $C_l = \min\{C_*, \hat{C}\}$.
Recall that $\rho$ is the strong convexity constant of $f$.
It can be shown that if $\eta$ is sufficiently small as a function of these constants (e.g. $\eta = \frac{\rho}{10C_u}$), let $\sigma^2 = \frac{C_l\eta}{2}$, then
\begin{equation}\label{eq:term2_bdd}
\begin{split}
    &\E_{\vx,\vx'}\frac{\left(\max\{\hketa, \tketa\} + q\right)^4}{q^2}
    \leq O\left(\pi(\vz_*)\left(\frac{1}{\rho\eta^2}\right)^{d/2}\right)
\end{split}
\end{equation}
The proof applies lemma \ref{lem:kernel_bdd} and the strong convexity of $f$ to simplify the expression with a Gaussian-integral like calculation; the details are deferred to appendix \ref{appendix:thm_p_eta}.

Plugging this inequality back in \eqref{eq:loss_ker} gives the statement of the theorem. 
\end{proof}
The proof of theorem \ref{thm:orig} can be adapted straightforwardly to accommodate the boundedness assumptions in theorem \ref{thm:kl}, namely, when $\frac{\hketa}{\keta}$ and $\frac{\keta}{\q}$ are bounded by $(\Delta_{\min}, \Delta_{\max})$ and $[\frac{1}{c_q}, c_q]$ respectively. In this case, the right hand side of equation \ref{eq:term2_bdd} will be updated to $(2c_q+1)^4 \E_{\vx,\vx'} (\keta)^2$.
The exponential dependency in $\eta$ is however still present, a consequence of lemma \ref{lem:kernel_bdd}.



\section{Conclusion}
\label{sec:conclusion}
We study contrastive learning objectives in time-series settings---particularly, when the data comes from a strong-mixing diffusion process. We provide both sample complexity bounds and quantitative results on the proximity of the learned transition kernel, given a good classifier for a judiciously chosen contrastive task. 

This is a first-cut work, and many natural open problems remain. For instance, how do other objectives (e.g. cross-entropy loss) perform? Are there better contrastive objectives than the proposed one that have a better scaling with dimension? Can we analyze the algorithmic effects of different choices of the contrast distribution $q$? Finally, can we analyze more complicated (e.g. latent-variable) diffusion processes using similar methods?  

\section*{Acknowledgements}

We acknowledge the support of NSF via IIS-1909816 and OAC-1934584.

\bibliography{ref}
\bibliographystyle{abbrvnat}

%

%


\appendix 

\onecolumn

\section{Omitted proofs}
\subsection{Omitted proof of Lemma \ref{lem:beta_pair} (mixing coefficient)}
\label{appendix:lem_mix_coeff}

We finish the proof of lemma \ref{lem:beta_pair} by proving $\beta_{\trPz}(2i\eta) = \beta_{\trSz}(2i\eta)$:
\begin{equation}
\begin{split}
    &\ \beta_{\trPz}(2i\eta) = 
     \frac{1}{2}\int \Big|\pi(\vz_0, \vz_0')\pi(\vz_{2i\eta}, \vz_{2i\eta}')\\
     &\ \ \ \ - \pi(\vz_0, \vz_0')p(\vz_{2i\eta}, \vz_{2i\eta}'|\vz_0, \vz_0')\Big|
    \\
    =&\ \frac{1}{2}\int \pi(\vz_0, \vz_0') \cdot \left|\pi(\vz_{2i\eta}, \vz_{2i\eta}') - p(\vz_{2i\eta}, \vz_{2i\eta}'|\vz_0, \vz_0')\right|
    \\
    =&\ \frac{1}{2}\int \left(\frac{1}{2} \pi(\vz_0)\left(p(\vz_0'|\vz_0) + \pi(\vz_0')\right)\right)\\
    &\ \cdot \left(\frac{1}{2} \big|\pi(\vz_{2i\eta}) - p(\vz_{2i\eta}|\vz_0)\big| \big(p(\vz_{2i\eta}'|\vz_{2i\eta}) + \pi(\vz_{2i\eta}'))\right)
    \\
    =&\ \frac{1}{8}\int_{\vz_0, \vz_{2i\eta}} \pi(\vz_0) \big|\pi(\vz_{2i\eta}) - p(\vz_{2i\eta}|\vz_0)\big| \\
    \cdot& \int_{\vz_0'} \left( p(\vz_0'|\vz_0) + \pi(\vz_0')\right) \cdot \int_{\vz_{2i\eta}'} \left(p(\vz_{2i\eta}'|\vz_{2i\eta}) + \pi(\vz_{2i\eta}')\right)
    \\
    =&\ \frac{1}{2}\int_{\vz_0, \vz_{2i\eta}} \pi(\vz_0)  \big|\pi(\vz_{2i\eta}) - p(\vz_{2i\eta}|\vz_0)\big|
    \\
    =&\ \frac{1}{2}\int_{\tilde{\vz}_0, \tilde{\vz}_{2i\eta}} \pi(\tilde{\vz}_0)  \big|\pi(\tilde{\vz}_{2i\eta}) - p(\tilde{\vz}_{2i\eta}|\vz_0)\big|
    = \beta_{\trSz}(2i\eta).
\end{split}
\end{equation}
\subsection{Omitted calculations for Theorem \ref{thm:opt_gen_bdd} (sample complexity)}
\label{appendix:thm_sample_cmplx}




We now provide the calculation details for Theorem \ref{thm:opt_gen_bdd}.

By lemma \ref{lem:gen_bdd}, and recall the empirical Rademacher complexity is $\rdm_{\mu} = O\left(\rk\sqrt{\log\mu/\mu}\right)$, we need to choose $T, \mu$ such that
\begin{equation}
     C\sqrt{\frac{1}{\mu}} \left(\rk\sqrt{\log \mu} + \sqrt{-\log\left(\delta - \Delta_{appr}\right)} \right) \leq \Delta_{gen}
\end{equation}
where $\Delta_{appr}^\mu := O\left(\frac{B}{\sqrt{T}}\mu^{\frac{3}{2}}\right)$ by lemma \ref{lem:beta}.

We would like to control $\Delta_{appr}^\mu = O(\delta)$.
Substituting in the choice of $T = \Omega\left(\frac{B^2\rk^3}{\delta^2\Delta_{gen}^3}\left(\log\frac{1}{\delta}\right)^{\frac{3}{2}}\right)$, we have
\begin{equation}
\begin{split}
    \Delta_{appr}^\mu = O\left(\frac{B\delta\Delta_{gen}^{3/2}}{B\rk^{3/2}}\left(\log\frac{1}{\delta}\right)^{-3/2}\cdot\mu^{3/2}\right) = O(\delta)
\end{split}
\end{equation}
which is satisfied by setting $\mu = \Theta\left(\frac{\rk\sqrt{\log(1/(\delta - \Delta_{appr}))}}{\Delta_{gen}}\right)$.

\subsection{Omitted calculations of Theorem \ref{thm:orig} (guarantee on \texorpdfstring{$p^{\eta}$}{p-eta})}
\label{appendix:thm_p_eta}

Recall that $c_*,C_*$ and $\hat{c}, \hat{C}$ are the constants in lemma \ref{lem:kernel_bdd} for $\tketa$ and $\hketa$ respectively. Denote $c := \max\{c_*, \hat{c}\}$, $C_u = \max\{C_*, \hat{C}\}$, $C_l = \min\{C_*, \hat{C}\}$.

We now show the omitted calculations for equation \ref{eq:term2_bdd} in the proof of Theorem \ref{thm:orig}.
\begin{equation}\label{eq:term2_bdd_appendix}
\begin{split}
    &\E_{\vx,\vx'}\frac{\left(\max\{\hketa, \tketa\} + q\right)^4}{q^2}
    \\
    \leq& \E_{\vx,\vx'} Z_\sigma^2 \exp\left(\frac{1}{\sigma^2}\|\vx-\vx'\|^2\right)\cdot \left(\frac{c}{\eta^{d/2}}\exp(C_u\eta\|\vx\|^2) + \frac{1}{Z_\sigma}\right)^4\exp\left(-\frac{4}{C_l\eta}\|\vz-\vx'\|^2\right)
    \\
    \leq& \E_{\vx} Z_\sigma^2\left(\frac{c}{\eta^{d/2}}\exp(C_u\eta\|\vx\|^2) + \frac{1}{Z_\sigma}\right)^4 \E_{\vx'} \exp\left(-\frac{2}{C_l\eta}\|\vx-\vx'\|^2\right)
    \\
    \leq& \E_{\vx} \frac{cZ_\sigma^2}{\eta^{d/2}}\left(\frac{c}{\eta^{d/2}}\exp(C_1\eta\|\vx\|^2) + (\pi C_l\eta)^{-\frac{d}{2}}\right)^4 \exp\left(C_u\eta\|\vx\|^2\right) \int_{\vx'} \exp\left(-\frac{3}{C_l\eta}\|\vx-\vx'\|^2\right)
    \\
    \leq& \frac{cZ_\sigma^2}{\eta^{5d/2}}\left(\frac{2\pi C_l\eta}{3}\right)^{\frac{d}{2}} \E_{\vx} \left(c\exp(C_u\eta\|\vx\|^2) + (\pi C_l)^{-\frac{d}{2}}\right)^4 \exp\left(C_u\eta\|\vz\|^2\right)
    \\
    \leq&
    c\left(\frac{2\pi^3C_2^3}{\eta^2}\right)^{d/2}
    \exp(-f(\vx_*)) \int_{\vx} \left(c\exp(C_u\eta\|\vx\|^2) + (\pi C_l)^{-\frac{d}{2}}\right)^4 \exp\left(-\left(\frac{\rho}{2}- C_u\eta\right)\|\vx\|^2\right)
    \\
    \leq&
    16c\pi(\vz_*)\left(\frac{2\pi^3C_l^3}{\eta^2}\right)^{d/2}
    \left[c^4 \left(\frac{\rho}{2} - 5C_1\eta\right)^{-\frac{d}{2}}
    + (\pi C_l)^{-2d}\left(\frac{\rho}{2} - C_u\eta\right)^{-\frac{d}{2}}\right]
    \\
    \leq& O\left(\pi(\vz_*)\left(\frac{1}{\rho\eta^2}\right)^{d/2}\right).
\end{split}
\end{equation}
where $C_1, C_2$ are constants introduced to simplify the notations.

\end{document}